\documentclass[10pt, twoside]{article}
\usepackage{lipsum} 
\usepackage[normalem]{ulem}
\useunder{\uline}{\ul}{}
\usepackage{amsmath,units}
\usepackage{pgfplots}
\usepackage{amsmath} 
\usepackage{tikz}
\usepackage[sc]{mathpazo} 
\usepackage[T1]{fontenc} 
\usepackage{microtype} 
\pgfplotsset{compat=1.6}
\usepackage{mathrsfs}
\usepackage{authblk}
\usepackage[hmarginratio=1:1,top=32mm,bottom = 24mm, columnsep=20pt]{geometry} 
\usepackage{multicol} 
\usepackage[hang, small,labelfont=bf,up,textfont=it,up]{caption} 
\usepackage{booktabs} 
\usepackage{float} 
\usepackage{hyperref} 

\usepackage{lettrine} 
\usepackage{paralist} 

\usepackage{abstract} 

\usepackage{titlesec} 
\titleformat{\section}[block]{\scshape\centering}{\thesection.}{1em}{} 
\titleformat{\subsection}[block]{\normalsize}{\thesubsection.}{1em}{} 
\usepackage{fancyhdr} 
\usepackage{graphicx}
\restylefloat{figure}
\usepackage{url}
\usepackage{amssymb}
\newcommand{\jeff}[1]{\underline{\textbf{#1}}}

\usepackage[citestyle=numeric-comp]{biblatex}
\usepackage{amsthm}

\newtheorem{theorem}{Theorem}[section]

\newtheorem{lemma}[theorem]{Lemma}
\theoremstyle{definition}
\newtheorem{definition}{Definition}[section]
\theoremstyle{remark}

\usepackage{algorithm}
\usepackage[noend]{algpseudocode}

\title{\vspace{-10mm}\fontsize{24pt}{12pt}\selectfont AI Reasoning Systems: PAC and Applied Methods} 
\author[1]{Jeffrey Cheng}

\begin{document}

\maketitle 

\thispagestyle{fancy} 

\tableofcontents

\begin{abstract}
    Learning and logic are distinct and remarkable approaches to prediction.  Machine learning has experienced a surge in popularity because it is robust to noise and achieves high performance; however, ML experiences many issues with knowledge transfer and extrapolation.
    
    In contrast, logic is easily intepreted, and logical rules are easy to chain and transfer between systems; however, inductive logic is brittle to noise.  We then explore the premise of combining learning with inductive logic into AI Reasoning Systems.
    
    Specifically, we summarize findings from PAC learning (conceptual graphs, robust logics, knowledge infusion) and deep learning (DSRL, $\partial$ILP, DeepLogic) by reproducing proofs of tractability, presenting algorithms in pseudocode, highlighting results, and synthesizing between fields.  We conclude with suggestions for integrated models by combining the modules listed above and with a list of unsolved (likely intractable) problems.
\end{abstract}

\newpage
\section{Introduction}
\subsection{Motivation}

Deep learning has been wildly successful in many domains within NLP, computer vision, and generative models.  However, there is no shortage of opportunities for deep learning to improve: in particular, 6 issues are recurringly identified in the literature.

\begin{enumerate}
    \item \jeff{Deep learning requires too much data.}
    
    Deep learning cannot learn from abstract definitions and requires tens of thousands of examples (humans are far more example-efficient with complex rules). \cite{lake-human}  Geoffrey Hinton has expressed concerns that CNNs face "exponential inefficiencies." \cite{sabour-capsule}
    \item \jeff{Deep learning has no interface for knowledge transfer.}
    
    There is no clean way to encode and transfer patterns learned by one neural net to another, even if their problem domains are very similar.
    
    \item \jeff{Deep learning handles non-flat (i.e. hierarchical) data poorly.}
    
    Neural nets take as input a flat vector.  Experts must specially encode tree, graph, or hierarchical data or (more likely) abandon deep learning.
    \item \jeff{Deep learning cannot make open inference about the world.}
    
    Neural nets perform poorly on extrapolation and multi-hop inference.
    \item \jeff{Deep learning integrates poorly with outside knowledge (i.e. priors).}
    
    Encoding priors (e.g. objects in images are shift-invariant) requires re-architecting the neural net (e.g. convolutional structure).
    \item \jeff{Deep learning is insufficiently interpretable.}
    
    The millions of black-box parameters are not human-readable, and it is difficult to ascertain what patterns are actually learned by the algorithm.
\end{enumerate}

\begin{table}[H]
\centering
\caption{Authors cited in this review who have identified these 6 issues.}
\begin{tabular}{l|l|l|l|l|l|l}
                          & Marcus \cite{marcus-appraisal}       & Bottou \cite{bottou-reasoning}       & Garnelo \cite{garnelo-symbolic} & Evans \cite{evans-explanatory} & Zhang \cite{zhang-like-humans} & Valiant \cite{valiant-robust}      \\ \hline
Needs too much data    & $\checkmark$ & $\checkmark$ & $\checkmark$   & $\checkmark$ &              &              \\ \hline
Can't transfer-learn  & $\checkmark$ & $\checkmark$ &                & $\checkmark$ &              &              \\ \hline
Can't encode hierarchy    & $\checkmark$ & $\checkmark$ &                &              & $\checkmark$ &              \\ \hline
Can't openly infer & $\checkmark$ &              &                &              &              & $\checkmark$ \\ \hline
Can't integrate priors    & $\checkmark$ & $\checkmark$ & $\checkmark$   &              &              &              \\ \hline
Uninterpretable           &              &              & $\checkmark$   & $\checkmark$ & $\checkmark$ &              \\
\end{tabular}
\end{table}

The first insight by Valiant is that inductive logic has none of these issues: logic is data-efficient, has simple rules for encoding/synthesis/structure, easily chains to make inferences, and has a human-readable symbology.  By combining logic and learning methods, one might achieve both the impressive, noise-tolerant performance of deep learning and the ease-of-use of inductive logic.  In this review, we call such a system an artificial intelligence reasoning system.

This literature review serves 3 primary purposes.
\begin{itemize}
    \item Summarize the findings from logic, mathematics, statistics, and machine learning on artificial reasoning systems.
    \item Rephrase the findings from disparate fields in the same terminology.
    \item Synthesize findings: find areas of useful integration and identify next steps.
\end{itemize}

\subsection{Terminology}

Since findings related to AI reasoning systems come from so many disparate fields, proofs have been written in many notations (e.g. in query form \cite{yang-diff-kb} and in IQEs \cite{valiant-robust}) -- we will rephrase all logic with  inductive logic programming (ILP) notation.
\begin{table}[H]
\centering
\caption{Terminology in the ILP framework}
\label{my-label}
\def\arraystretch{1.5}
\begin{tabular}{|l|l|}
\hline
Name            & Notation                                                                                                                                                                                                                                                       \\ \hline
Atom            & $\alpha\in A$                                                                                                          \\ \hline
Ground Atom     & $g\in G$                                                                                                              \\ \hline
Definite clause & $r=\left(\alpha\leftarrow \alpha_1,...,\alpha_m\right)$                                                                                                                                          \\ \hline
Consequence     & $C=\bigcup_{i=1}^\infty C_i$, \newline$C_i=C_{i-1}\bigcup \{\gamma \mid \gamma \leftarrow \gamma_1,...,\gamma_m\in C_{i-1}\}$                                                                        \\ \hline
ILP Problem     & $I=(G, P, N)$                                                                                                                                                            \\ \hline
ILP Solution    & $R\text{ s.t. }B, R\models \gamma\iff \gamma\in P\setminus N$                                                                    \\ \hline
Knowledge Base    & $KB$                                                                   \\ \hline
\end{tabular}
\end{table}                                                                     

\theoremstyle{definition}
\begin{definition}{(Atom)}
An $n$-ary predicate $p(t_1,...,t_n)$ where each term $t_i$ is a variable or constant.  Typically the truth value of a relation on given args
\end{definition}

\theoremstyle{definition}
\begin{definition}{(Ground Atom)}
An atom with no variables                                                                                                                          
\end{definition}

\theoremstyle{definition}
\begin{definition}{(Definite clause)}
A rule where head $\alpha$ is implied by each body atom $\alpha_i$ being true 
\end{definition}

\theoremstyle{definition}
\begin{definition}{(Consequence)}
The closure over rules $R$ and ground atoms $G$.
\end{definition}

\theoremstyle{definition}
\begin{definition}{(ILP Problem)}
A 3-tuple of a set of ground atoms, a set of positive examples, and a set of negative examples. 
\end{definition}

\theoremstyle{definition}
\begin{definition}{(ILP Solution)}
A set of definite clauses that includes all positive examples in its consequences and no negative examples. 
\end{definition}

\theoremstyle{definition}
\begin{definition}{(Knowledge Base)}
A set of known or learned rules.
\end{definition}

In this review, all standard definitions of PAC learning from K\&V apply.  The term PAC-learning here means \textit{efficiently} PAC-learnable (i.e. an algorithm exists that accomplishes $\epsilon$-accuracy with confidence $1-\delta$ in $Poly\left(\frac{1}{\epsilon}, \frac{1}{\delta}\right)$ examples and polynomial time.

\newpage
\section{PAC Frameworks}
\subsection{Unconstrained Reasoning}

In order to motivate the ILP framework, we present first a negative result: learning a reasoning system is \textbf{intractable in the PAC setting} if the outputs of rules are relaxed from boolean outputs $\{0,1\}$ to arbitrary outputs.

\subsubsection{Learning Conceptual Graphs}

\theoremstyle{definition}
\begin{definition}{(Conceptual Graph)}
\textbf{As originally formulated}, a simple conceptual graph is a directed bipartite graph $G=(X, Y, E)$, where nodes $x\in X$ correspond to relations, nodes $y\in Y$ correspond to ``concepts,'' and edges symbolize an ordered relationship between a predicate and one of its arguments. \cite{croitoru-conceptual}  \textbf{An identical formulation in ILP vocabulary} is $G=(X,Y,E)$, where nodes in $X$ are head atoms, nodes in $Y$ are body atoms (relaxed from booleans to arbitrary objects), and the in-edges of a node represent a rule.
\end{definition}

\begin{figure}[H]
  \caption{An example of a simple conceptual graph}
  \centering
    \includegraphics[width=0.5\textwidth]{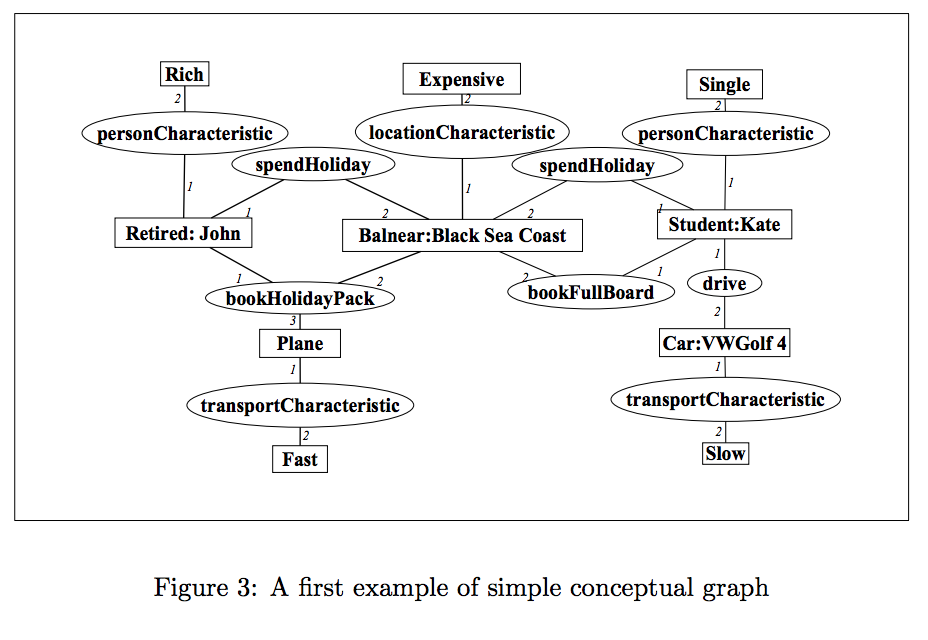}
\end{figure}

\theoremstyle{definition}
\begin{definition}{(Projection)}
Informally, if we are given conceptual graphs $G$ and $H$, then $G$ is more general than $H$ (denoted $G\geq H$, e.g. human $\geq$ student) iff there exists a projection from $G$ to $H$.  Formally, a projection from $G$ to $H$ is a mapping $\Pi: V_X(G)\cup V_Y(G)\rightarrow V_X(H)\cup V_Y(H)$ such that:
\begin{align*}
    &\Pi(V_X(G))\subseteq V_X(H)\text{ and }\Pi(V_Y(G))\subseteq V_Y(H)\\
    &\forall x\in V_X(G), \forall y\in V_Y(G), x\in N_G^i(y)\implies \Pi(x)\in N_H^i\left(\Pi(y)\right)\tag{$G$ dominates $H$ in distance}\\
    &\forall v\in V_X(G)\cup V_Y(G), E(H)\subseteq E(G)\tag{\cite{croitoru-conceptual}}\\
\end{align*}
\end{definition}

We present a simple intermediate result due to Baget and Mugnier \cite{baget-np}.
\begin{lemma}
Deciding whether $G\geq H$ is an NP-complete problem.
\end{lemma}
 
\begin{proof}
The last condition alone suffices for the proof.  We want to find a mapping $\Pi: V(G)\rightarrow V(H)$ such that $(a,b)\in E(G)\implies \left(\Pi(a), \Pi(b)\right)$.

Consider the special case where $H$ is a graph where the only edges form a $k$-clique.  In this special case, we want to decide whether the edges of $G$ have some mapping to a $k$-clique.  This is now the decision problem for the $k$-clique problem, which is known to be NP-complete.

Since the projection problem is reduced to an NP-complete problem in this special case, we conclude that the projection problem is NP-complete as well.
\end{proof}

We now arrive at the desired negative result due to Jappy and Nock. \cite{jappy-conceptual}
\begin{theorem}
Let $CG$ be the set of conceptual graphs, and let $C\subseteq CG$ be a concept class of such graphs.  $C$ is not PAC-learnable unless $NP\subseteq P/Poly$.
\end{theorem}
\begin{proof}
Let $d=inf(c)$ denote the most specific conceptual graph that has the same representation as $c$.  Let $D=\{inf(c) \forall c\in C\}$.

Deciding whether $inf(x)$ is in $D$ for arbitrary $x\in C$ is a special case of the projection problem.  By Lemma 2.1, testing membership for $inf(x)\in D$ is NP-complete.

We cite Theorem 7 from Schapire's \textit{The strength in weak learning} \cite{schapire-weak}:
\begin{quote}
Suppose $C$ is learnable, and assume that $X_n=\{0,1\}^n$.  Then there exists a
polynomial $p$ such that for all concepts $c\in C_n$ of size $s$, there exists a circuit of size $p(n, s)$ exactly computing $c$.
\end{quote}
We then have that if $D$ is PAC-learnable, then learning $D$ can be computed in $P/Poly$.  The contrapositive of this theorem is that if learning $D$ is not $ P/Poly$, then $D$ is not PAC-learnable.

We know that $D$ is NP-complete and take as assumption that $NP\neq P/Poly$.  

Therefore, $D$ (and hence $C$) is not PAC-learnable.
\end{proof}

The consequence of this negative result is that all other discussion will be about constrained reasoning: either reasoning about small sets of objects/relations or reasoning about arbitrary sets of objects/relations with rules of constrained form.
\subsection{Constrained Reasoning}
\subsubsection{Learning to Reason}

Consider the problem of finding a satisfying assignment given $W\models \alpha$, where $W$ is a boolean function and $\alpha$ is a CNF.

It is well-known that this problem is NP-Hard.

\theoremstyle{definition}
\begin{definition}{(Reasoning query oracle)}
Define reasoning query oracle $RQ(A, f ,Q)$ as an oracle that picks an arbitrary query $\alpha\in Q$, then checks whether learning agent $A$ is correct about its belief of $f\models \alpha$.  If $A$ is correct, $RQ$ returns True; else, $RQ$ returns a counterexample.
\end{definition}

\theoremstyle{definition}
\begin{definition}{(Exact Learn-to-Reason)}
Let $F$ be a class of boolean functions, and let $Q$ be a query language.  An algorithm $A$ is an exact Learn-to-Reason algorithm for the problem $(F,Q)$ iff $A$ learns in polynomial time with a reasoning query oracle, and after the learning period, $A$ answers True to query $\alpha$ iff $f\models \alpha$. 
\end{definition}

If we are given that $W$ has a polynomial size DNF, $\alpha$ has at most $\log n$ literals in each clause, and access to an equivalence query oracle, an example oracle, and a reasoning query oracle, then \textbf{we can find an exact reasoning solution to the NP-Hard CNF problem}.

We present an exact learn-to-reason algorithm for learning some boolean function $f$ with polynomial-size DNF over basis $B$.

\begin{algorithm}[H]
\caption{Learn-to-Reason Training Algorithm}\label{alg:l2r-train}
\begin{algorithmic}[1]
\Procedure{Learn-to-Reason-Train}{$B, F, RQ$}\Comment{Learning CNFs via DNF from Below}
\State{For all $b\in B$, $\Gamma_b\gets\varnothing$}
\While{Grace Period}
    \If{$\neg RQ(\Gamma, F, Q)$}
        \State Let $x$ be the returned counterexample.
        \State Iterate through $B$ until encountering some $b\in B$ s.t. $x\in M_b(\Gamma_b)$.
        \State $\Gamma_b\gets Gamma_b\cup x$
        \State Repeatedly call $MQ(f)$ to find a new minimal model of $f$ w.r.t. $b$.
    \EndIf
\EndWhile
\EndProcedure
\end{algorithmic}
\end{algorithm}

The answering algorithm follows given input query $\alpha$ and $\Gamma=\bigcup_{b\in B} \Gamma_b$ from the learning algorithm; simply evaluate the CNF $\Gamma$ on the boolean vector $\alpha$.

\begin{theorem}
The Learn-to-Reason Training Algorithm is mistake-bound, learns $f$ from below, and is an exact Learn-to-Reason algorithm.
\end{theorem}

\begin{proof}

Denote $h=\bigwedge_{b\in B} \left(\bigvee_{z\in Gamma_b} M_b(z)\right)$

We begin by noting that the algorithm never makes a mistake when it returns False since $\Gamma\subseteq f$.  Thus, it makes mistakes only on positive counterexamples; since $x\in f\setminus h$ is bounded (in fact, polynomial), the number of mistakes is bounded.

We cite 2 theorems by Khardon and Roth \cite{khardon-l2r}:

\begin{enumerate}
    \item For boolean function class $F$ and query language $Q$ with basis $B$, the least upper bound $f$ is well-defined as $f_{lub}=\bigwedge_{b\in B} M_b(f)$.
    \item Given $KB\in F$, $\alpha\in Q$, and basis $B$, $KB\models \alpha$ iff $\forall u\in \Gamma_KB^B$, $\alpha(u)=1$.
\end{enumerate}

From the former, we have that $h$ is the least upper bound on $f$ after $|\Gamma_f^B|$ calls to the oracle $RQ(f,Q)$ by definition of a least upper bound.

The postcondition of the latter is satisfied by the algorithm, so we conclude that the output of the algorithm $h=f_{lub}$ is a solution to the learn-to-reason version of the CNF problem.
\end{proof}

Thus, this algorithm learns to reason about CNF from a knowledge base of counterexamples efficiently. \cite{khardon-l2r}  We note that this has not solved an NP-hard problem in the traditional sense; the additional reasoning power comes from the oracle $RQ$ and its ability to provide counterexamples in constant time.  In fact, the problem is still traditionally intractable with our additional restriction on the DNF size; it is NP-hard to find a satisfying assignment for a CNF even given that there exists only one satisfying assignment. \cite{valiant-np}

Finally, we note that there is an extension from exact learning-to-reason to the PAC setting, as defined by Khardon and Roth.

\theoremstyle{definition}
\begin{definition}{(Fair)}
Given boolean function $f$, query $\alpha\in Q$, and $\epsilon>0$, query $\alpha$ is $(f,\epsilon)$-fair if $\Pr_D[f\setminus \alpha]=0$ for $\Pr_D[f\setminus\alpha]>\epsilon$; either the query does not occur (and thus no mistakes can be made over it) or it occurs frequently.
\end{definition}

\theoremstyle{definition}
\begin{definition}{(PAC Learn-to-Reason)}
Let $F$ be a class of boolean functions, and let $Q$ be a query language.  An algorithm $A$ is an exact Learn-to-Reason algorithm for the problem $(F,Q)$ iff $A$ is polynomial time in $n, \frac{1}{\epsilon}, \frac{1}{\delta}$ with a reasoning query oracle, and after the learning period, $A$ answers True to query $\alpha$ iff $f\models \alpha$ for a $(f,\epsilon)$-fair query with probability $1-\delta$. 
\end{definition}

This definition will serve as an inspiration for the next finding in Robust Logics, where the PAC-setting is critical to the tractability of the rule-learning.

\subsubsection{Robust Logics}

\begin{table}[H]
\centering
\caption{Valiant's definitions and their corresponding equivalents}
\label{my-label}
\begin{tabular}{|l|l|p{10cm}|}
\hline
Valiant             & ILP         & Notes                                                                                                                                                                                         \\ \hline
Object              & Variable    &                                                                                                                                                                                          \\ \hline
Token               & N/A         & A named reference to an object.,Valiant himself assumed distributional symmetry to avoid dealing with permutations of tokens w.r.t. objects, so we omit the term and concept altogether. \\ \hline
Relation            & Atom        &                                                                                                                                                                                          \\ \hline
Deduction           & Consequence &                                                                                                                                                                                          \\ \hline
Connective function & Predicate   & Valiant assumes these connective functions belong to a PAC-learnable class (e.g. linear threshold functions).                                                                                                                                                                                         \\ \hline
Rule                & Rule        & Valiant's rules have a slightly different structure: $r:= Q(A)\implies [f(R_{i_1}, ...,R_{i_l})\equiv R_{i_0}]$, where $Q$ is an arbitrary boolean function on the object set and $f$ is a connective function.                                                                 \\ \hline
\end{tabular}
\end{table}

\theoremstyle{definition}
\begin{definition}{(Scene)}
Given objects $A$, relations $R=\{R_1,...,R_t\}$, and arities $\alpha=\{\alpha_1, ..., \alpha_t\}$, a scene $\sigma$ is a vector of length $\sum_{i=1}^t n^{\alpha_i}$, where $\sigma[j]$ is the truth value of the $j$th relation.
\end{definition}

\theoremstyle{definition}
\begin{definition}{(Obscured Scene)}
An obscured scene $\sigma^\$$ is a scene whose vector elements can take on 3 values: $\{0, 1, \$\}$, where $\$$ denotes an obscured truth value. 
\end{definition}

In Valiant's setting of Robust Logics, we place constraints on the atoms themselves; if predicates are restricted to boolean functions that are PAC-learnable, then \textbf{any class of rules is PAC-learnable} from examples (even with partial knowledge).  We then present 2 algorithms:
\begin{itemize}
    \item A learning algorithm for inducing rules from examples.
    \item A deduction algorithm using learned rules to predict the truth values of obscured variables.
\end{itemize}

\paragraph{PAC-learnability from Scenes}

\begin{theorem}
If the class of connective functions $C^l$ is PAC-learnable to accuracy $\epsilon$ and confidence $1-\delta$ by algorithm $M$ in $L(l, \epsilon, \delta)$ examples, then any class of rules with constant arity over $C, R, A$ is PAC-learnable from scenes in $L(l, \epsilon, \delta)$ examples.
\end{theorem}

\begin{proof}
The proof borders on tautological.  We target a rule $q:= Q(A)\implies [f(R_{i_1}, ...,R_{i_l})\equiv R_{i_0}]$

Note that a particular scene $\sigma$ yields an input vector of $l$ truth values and an output boolean.  Suppose we sample $L(l, \epsilon, \delta)$ such scenes.  By assumption, algorithm $M$ can learn some function $f'\in C$ that is $\epsilon$-accurate with confidence $1-\delta$.

If the distribution of scenes is $D$, then the probabilistic guarantee is:
\begin{align*}
    er_D(f')&=\sum_{\sigma} D(\sigma)\left(err^+(f',\sigma) + err^-(f', \sigma)\right)\leq \epsilon
\end{align*}

We define a rule $q':=Q(A)\implies \left[f'(R_{i_1}, ...,R_{i_l})\equiv R_{i_0}\right]$ ($q$, but with $f$ substituted with $f'$).

\begin{align*}
    er_D(q')&=\sum_{\sigma} D(\sigma)\left(err^+(q',\sigma) + err^-(q', \sigma)\right)\\
    &=\sum_{\sigma} D(\sigma)\left(err^+(f',\sigma) + err^-(f', \sigma)\right)\leq \epsilon
\end{align*}

\end{proof}

\paragraph{Learning Algorithm}

We must be given some set of templates for rules of the form:
\begin{align*}
    r:=Q(A)\implies \left[f(R_{i_1},...,R_{i_l})\equiv R_{i_0}\right]
\end{align*}
where only $f$ is unknown, such as the following:
\begin{align*}
    \forall x_1, x_2, x_3 Q(x_1, x_2, x_4, \sigma)\implies \left[f(\exists y_1 R_1(x_1,x_2,y_1),\exists y_2 R_2(x_2,x_3,y_2)\equiv R_1(x_1,x_2,x_3)\right]
\end{align*}
Let $T$ be a set of such rule templates.  The algorithm constructs a training set for PAC-algorithm $M$ over the connective functions and runs it $|T|$ times over the rule templates.

\begin{algorithm}[H]
\caption{Robust Logics Learning Algorithm}\label{alg:robust-learn}
\begin{algorithmic}[1]
\Procedure{Learn-Rules}{$\{\sigma\}, M, T$}\Comment{PAC-algorithm for learning rules}
\State{$S\gets []$}\Comment{Initialize Rule Set}
\For{$t\in T$}
    \State{Let $f_t$ be template $t$'s connective function.}
    \State{Let $R^p_t$ be the relations in the predicate of $T$, and let $R^h_t$ be the head.}
    \State{$T\gets[]$}\Comment{Initialize training set for $M, t$}
    \State{$L\gets[]$}\Comment{Initialize labels for $M, t$}
    \For{$\sigma\in\{\sigma\}$}
        \State{Let $\sigma[R^p_t]$ be $\sigma$ sliced along elements in $R^p_t$}
        \State{$T\gets T + \sigma[R^p_t]$}\Comment{Append predicates to training set.}
        \State{$L\gets L + R^h_t$}\Comment{Append head to label set.}
    \EndFor
    \State{$S\gets S + M(T, L)$}\Comment{Append rule learned from $M$ to ruleset}
\EndFor
\Return $S$
\EndProcedure
\end{algorithmic}
\end{algorithm}

\paragraph{Deduction Algorithm}

The Deduction Algorithm is significantly more complex because the deduction can yield both ambiguities and contradictions, the former arising from insufficient information and the latter from errors in the rules (to magnitude $\epsilon$).  We construct a graph $G=(V,E)$, where the vertices $V$ are relations and edge $(v_1,v_2)\in E$ iff there is a rule with $v_1$ on the LHS and $v_2$ on the RHS.  We present an illustrative algorithm for the case in which $G$ is acyclic, then a general algorithm for arbitrary graphs $G$.

\begin{algorithm}[H]
\caption{Acyclic Deduction Algorithm}\label{alg:acyclic-deduction}
\begin{algorithmic}[1]
\Procedure{Acyclic-Deduction}{$G$, $S$}\Comment{Evaluates relations from acyclic rules graph.}
\State{Let $Z$ be the topologically sorted order of the relations.}
\For{$R\in Z$}
    \State{Let $S_R\subseteq S$ be the set of rules with $R$ as the head.}
    \If{$S_R=\varnothing$}\Comment{$R$ is a leaf.}
        \State{Assign 0/1 to $R$}
    \ElsIf{All of the outputs of $S_R$ are true.}
        \State{Assign 1 to $R$}
    \ElsIf{All of the outputs of $S_R$ are false.}
        \State{Assign 1 to $R$}
    \ElsIf{Two rules in $S_R$ lead to a contradiction}
        \State{Assign ? to $R$}
    \Else
        \State{Assign 0/1 to $R$}
    \EndIf
\EndFor
\Return
\EndProcedure
\end{algorithmic}
\end{algorithm}

The acyclic case is quite intuitive since there is an order of evaluation.  The general case has no such ordering and so must proceed for an indeterminate (but certainly finite) amount of time.

\begin{algorithm}[H]
\caption{General Deduction Algorithm}\label{alg:acyclic-deduction}
\begin{algorithmic}[1]
\Procedure{GeneralDeduction}{$G$, $S$}\Comment{Evaluates relations from any rules graph.}
\State{Set all relations to ?}
\State{UpdateFlag $\gets$ True}
\While{Flag}
    \State{UpdateFlag $\gets$ False}
    \For{$R\in V$}
        \State{Let $S_R\subseteq S$ be the set of rules with $R$ as the head.}
        \If{Any relations in $S_R$ are ?}
            \State{Do nothing.}
        \ElsIf{$S_R=\varnothing$}\Comment{$R$ is a leaf.}
            \State{Assign 0/1 to $R$}
            \State{UpdateFlag $\gets$ True}
        \ElsIf{All of the outputs of $S_R$ are true.}
            \State{Assign 1 to $R$}
            \State{UpdateFlag $\gets$ True}
        \ElsIf{All of the outputs of $S_R$ are false.}
            \State{Assign 1 to $R$}
            \State{UpdateFlag $\gets$ True}
        \ElsIf{Two rules in $S_R$ lead to a contradiction}
            \State{Assign ? to $R$}
            \State{UpdateFlag $\gets$ True}
        \Else
            \State{Assign 0/1 to $R$}
            \State{UpdateFlag $\gets$ True}
        \EndIf
    \EndFor
\EndWhile
\Return
\EndProcedure
\end{algorithmic}
\end{algorithm}

We would like for this general algorithm to be label a large proportion of relations and to achieve small error on those relations that it labels.  We formalize these notions as \textbf{soundness} and \textbf{completeness}.

\theoremstyle{definition}
\begin{definition}{(PAC-Sound)}
A deduction algorithm is PAC-sound if:
\begin{enumerate}
    \item it is polynomial time.
    \item for any error $\epsilon>0$,$\frac{\epsilon}{||S||\cdot |A|}$, where $||S||$ the description size of $S$ and $|A|$ is the number of variables.
\end{enumerate}
\end{definition}

\theoremstyle{definition}
\begin{definition}{(Deduction sequence)}
A deduction sequence is a sequence is triples $<(q_j, \pi_j, b_j>)$, where $q_j$ is a rule, $\pi_j$ is a mapping to variables/objects, and $b_j$ is the boolean output.
\end{definition}

\theoremstyle{definition}
\begin{definition}{(Validity)}
A deduction algorithm is valid for an input $\sigma$ for all unobscured relations $R_i$, each rule in the set of rules with $R_i$ as the RHS $S_i$ is satisfied (i.e. if the predicate $Q$ is true, then predicted boolean matches the unobscured truth value).
\end{definition}

\theoremstyle{definition}
\begin{definition}{(Complete)}
A deduction algorithm is complete for a class of rule sets $S$ with respect to $R,A,C$ iff there is a valid deduction sequence ending with $(q, \pi, b)$ iff the deduction algorithm outputs $b$ given inputs $q,\pi$.
\end{definition}

\begin{theorem}
The general deduction algorithm is PAC-sound and complete if the maximum arity of rules are upper-bounded by some constant $\alpha$.
\end{theorem}

\begin{proof}

For each rule $q$, we define $Ev(n,\alpha)$, $l(q)$ and $c(q)$ to be the number of evaluations per LHS predicate, the number of predicates in the LHS, and the complexity of evaluating a single connective function, respectively.

For each rule $q$, the algorithm runs at most $n^\alpha$ times.  We further note that the truth value of any predicate as currently bounded can be computed in $O(n^\alpha)$ steps.

Then the complexity of the algorithm is bounded by:
\begin{align*}
    \sum_q n^\alpha_q\left(Ev(n,\alpha)\cdot l(q)+c(q)\right)&\leq\sum_q n^{\alpha_q}\left(n^\alpha\cdot l(q)+c(q)\right)\\
    &\leq \sum_q n^\alpha\left(n^\alpha\cdot l(q)+c(q)\right)\tag{By arity bound}\\
    &\leq n^\alpha\left(||S||n^\alpha + |S|c(S)\right)\tag{By definition of description size}
\end{align*}
Thus, the runtime is polynomial for constant $\alpha$.

We now show that the algorithm meets the accuracy conditions.

For rule set $S$ and scene $\sigma$, the algorithm evaluates at most $|S|$ rules.  We previously bounded the error for a single rule to:
\begin{align*}
    er_D(q)&=\sum_\sigma D_(\sigma)\left(err^+(q,\sigma)+err^-(q,\sigma)\right)\leq \epsilon'\tag{By PAC-learnability from scenes}
\end{align*}

We select $\epsilon'=\frac{\epsilon}{|S|n^\alpha}$, where $\epsilon$ is the error required by the definition of PAC-soundness.  We see then that the probability that an arbitrary scene $\sigma$ errs is no more than $|S|n^\alpha\cdot \epsilon'=\epsilon$.
\end{proof}

Thus, we have now demonstrated an algorithm that learns to reason soundly and completely on given variables, relations, scenes each describing the truth values of some relations, knowledge of the structure of rules, PAC-learnability of the boolean functions in the rules, and bounds on the arities and sizes of the rules. \cite{valiant-robust}

\subsubsection{Knowledge Infusion}

\theoremstyle{definition}
\begin{definition}{(Knowledge Infusion)}
Any process of knowledge acquisition by a computer that satisfies the following 3 properties:
\begin{enumerate}
    \item The stored knowledge is encoded such that principled reasoning on the knowledge is computationally feasible.
    \item The stored knowledge and corresponding reasoning must be robust to errors in the system inputs, to uncertainty in knowledge, and to gradual changes in the truth.
    \item The acquisition must be carried out automatically and at scale.
\end{enumerate}
\end{definition}

Valiant and Loizos present motivations, PAC bounds, and an applied algorithm in the context of knowledge infusion.  We discuss their 2 seminal publications. \cite{valiant-infusion} \cite{michael-infusion}

\paragraph{Motivation}

The premise of knowledge infusion is to extend robust logics using cognitive science as inspiration.  Recall that robust logics learns rules between relations given scenes: booleans that give a partial description of variables in a specified ontology.  Contrast this with standard ILP frameworks, which quantify existentially and universally over an unlimited/ill-specified world.  This corresponds roughly to the notion of a \textbf{working memory}: learning on a small-dimensional, manageable subset of the world.

This notion spawns 3 subproblems.

\begin{enumerate}
    \item \jeff{Parallel Concepts}
    
    There is no longer a single target function or a target query; to infuse large quantities of knowledge, the complexity of learning many targets may increase.
    
    \item \jeff{Teaching materials}
    
    Learning from a working memory model requires an abundance of teaching materials to educate the model with an appropriate order of positive and negative examples.
    
    \item \jeff{Are Rules Necessary?}
    
    As we previously saw, Khardon and Roth's algorithm for learning-to-reason re-learns the model from scratch after each mistake by memorizing examples.  Rules may or may not provide better performance.
\end{enumerate}

We discuss the former two, demonstrate applied results, and then present arguments from Juba, Valiant, Khardon, and Roth on the final point.

\paragraph{Parallel Concepts}

We present Valiant's findings on learning many concepts simultaneously from few examples.  Surprisingly, learning $N$ functions rather than just one over the same domain does not require $N$ times as many examples; very few additional examples are needed.

\theoremstyle{definition}
\begin{definition}{(Simultaneous-PAC Learning Algorithm)}
For polynomial $m$, some $\epsilon>0$, some confidence $1-\delta$, and some number of functions $N$, $L$ is an $m(n,\epsilon,\delta,N)$-simultaneous PAC-learning algorithm over concept class $C$ iff for any distribution $D$ after $m(n,\epsilon,\delta,N)$ examples, it can learn a set of $N$ hypotheses for arbitrary concepts $c_1,...,c_N\in C$ to accuracy $\epsilon$ with probability $1-\delta$.
\end{definition}

We now present 2 theorems in support of the notion of a log-increase on the number of examples needed to support many functions.

\begin{theorem}
If there is an $m(n,\epsilon, \delta)$-PAC learning algorithm for concept class $C$, then there is an $m(n,\epsilon,\delta', N)$-simultaneous-PAC learning algorithm for $C$, where $m(n,\epsilon,\delta',N)=m(n,\epsilon,\delta'/N)$.
\end{theorem}

\begin{proof}
We sample $m(n,\epsilon,\delta)$ random examples from $D$ and apply $L$ to each function separately.  

The probability of a specific hypothesis having error more than $\epsilon$ is $\delta$.  By union-bound, the probability of at least one hypothesis having error more than $\epsilon$ is $N\delta$.

We choose $\delta'=N\delta$.  This completes the proof.
\end{proof}

A consequence of this algorithm is that concept classes that have a dependence on $\log\left(\frac{1}{\delta}\right)$ (such as disjunctions over $k$ variables \cite{littlestone}) then have a log-dependence on $N$, which is promising.

The next theorem extends this principle in general to all concept classses.

\begin{theorem}{(Simultaneous Occam Theorem)}
Suppose we have domain $X_n=\{x_1,...,x_n\}$, distribution $D_n$, and target functions $F=\{f_1,...,f_N\}$, and a learning algorithm $A$ that outputs hypothesis $h_1,...,h_N\in H_n$ for functions in $F$.

Suppose further that an experiment occurs where $m*$ examples are drawn from $D_n$, and the result is that the hypotheses $H*$ predict \textbf{every sample point correctly}.

Then for any $m>\frac{1}{2\epsilon}\left(\log_2 |H_n| + \log_2\left(\frac{1}{\delta}\right)\right)$ and any $\epsilon < \frac{1}{2}$:
\begin{enumerate}
    \item $\Pr[\text{There is a hypothesis with error } > \epsilon]<\delta$
    \item $\Pr[\text{There is a hypothesis with precision } > \frac{\epsilon}{\phi_i}]<\delta$, where $\phi_i$ is the predicted positive rate.
    \item $\Pr[\text{There is a hypothesis with recall } > \frac{\epsilon}{\phi_i-\epsilon}]<\delta$
\end{enumerate}
\end{theorem}

\begin{proof}
Suppose that a function $h\in H_n$ (not necessarily one of the chosen hypotheses) has error greater than $\epsilon$.  By Multiplication Rule, the probability of $h_i$ labeling all $m*$ examples correctly is $(1-\epsilon)^m$.

By union bound, the probability of \textbf{any} function in $H_n$ having error greater than $\epsilon$ is then $|H_n|(1-\epsilon)^m$.

We then have:
\begin{align*}
    |H-n|(1-\epsilon)^m &< \delta\tag{By PAC algorithm}\\
    \frac{1}{\delta}|H-n|(1-\epsilon)^m &<1 \\
    \log\left(\frac{1}{\delta}\right)+\log |H_n| +m\log(1-\epsilon)&<0\\
    m &> \boxed{\frac{1}{2\epsilon} \left[\log|H_n|+\log\left(\frac{1}{\delta}\right)\right]}
\end{align*}

This gives us the theorem's required lower bound for the number examples; the example bound is equivalent to the PAC condition.  We now prove the three error bounds.

The first error bound is immediate from the definition of PAC learning.

The second error bound follows from definitions:
\begin{align*}
    \text{Precision} &= 1 - \frac{\Pr[h(x)=1\land f(x)=0]}{\Pr[h(x)=1]}\tag{By definition}\\
    &= \boxed{1 - \frac{\epsilon}{\phi}}\tag{By definition}
\end{align*}

The third error bound follows from definitions:
\begin{align*}
    \text{Recall} &= 1 - \frac{\Pr[h(x)=0\land f(x)=1]}{\Pr[f(x)=1]}\tag{By definition}\\
    &= 1 - \frac{\epsilon}{\Pr[h(x)=1]-\Pr[f(x)=0\land h(x)=1]}\tag{Probability set difference}\\
    &=\boxed{1-\frac{\epsilon}{\phi-\epsilon}}
\end{align*}
\end{proof}

Valiant concludes that if learning a single rule is tractable in the PAC sense, then it is also tractable to learn many rules (even exponentially many) given that there is only a log dependence on the number of rules.  This arises from a re-use of examples for every rule.

\paragraph{Teaching Materials and Scene Construction}

In Michael and Valiant's empirical demonstration of knowledge infusion, they create ``teaching materials'' by turning text from news sources into scenes as per the Robust Logic framework.

The dataset is the \textit{North American News Text Corpus} \cite{graff-text}, comprising 6 months worth of articles (roughly 500000 sentences).  Sentences were annotated with the \textit{Semantic Role Labeler}, an automated tagger \cite{ccg-uiuc}; sentence fragments were then passed in the \textit{Collins Head Rules}, which extracts keywords to summarize the sentence. \cite{collins-head}

Finally, each sentence summary is converted to a scene as follows:
\begin{enumerate}
    \item Create an entity for each tokenized word with a semantic object association.
    \item Create a relation for each verb in the sentence on the relevant objects in the sentence.  Michael and Valiant cap the arity $\alpha$ of the relations at 2 since scene size is exponential in $\alpha$.
    \item Create a relation for \textit{proximity instances} for words that are close to each other in the sentence.
\end{enumerate}

A fuller description of scene construction can be found on pg. 381-382 of the original paper. \cite{michael-infusion}

\paragraph{Experimental Approach}

Given the constructed scenes, there are 3 primitive rule operations:
\begin{enumerate}
    \item \jeff{Rule Induction}: given a target relation and scenes, create a rule for the target relation.  Michael and Valiant choose to use the Winnow algorithm for lienar threshold functions (perceptrons) to learn the rules, as per the sugggestion from Valiant's Robust Logics paper. \cite{valiant-robust}
    \item \jeff{Rule Evaluation}: given an input rule $K_i$, the input rule's head relation $R_i$, and a set of scenes $E$, predict the truth value of $R_i$.
    \item \jeff{Rule Application}: given a set of rules $K$ and a set of scenes $E$, apply all rules $k\in K$ in parallel, and enhance the original set of scenes $E$ with unambiguous truth values determined by the rules.
\end{enumerate}

Michael and Valiant define the chaining of these 3 primitives as the \textbf{rule chaining task}.  They then create 4 discrete distributions\footnote{We refer to these as distributions, but the original terminology is ``experiment.'' The renaming is due to the term experiment being quadruply overloaded by the original authors.} for the rule chaining task.

Each of the following 4 distributions takes as input a training set $T_0$ of scenes, a testing set $E_0$ of scenes, an enhancement set $R_{enh}$ of relations (akin to a knowledge base or a prior of rules), and a target relation $R_{i_0}$.  The output in each task is a rule $K_{i_0}$ for predicting $R_{i_0}$, and we evaluate the performance of $K_{i_0}$ on the testing set $E_0$.

\begin{enumerate}
    \item \jeff{Distribution 00}: the enhancement set $R_{enh}$ is ignored.
    \item \jeff{Distribution 11}: the enhancement set $R_{enh}$ is used on the training set $T_0$ in learning rule $K_{i_0}$.  The enhancement set is further used to augment the testing scenes into $E_0'$.
    \item \jeff{Distribution 01}: the enhancement set is only used on the testing set.
    \item \jeff{Distribution 10}: the enhancement set is only used in learning.
\end{enumerate}

Finally, Michael and Valiant experiment along one more axis: syntactic vs. semantic information.
\begin{enumerate}
    \item \jeff{Syntactic}: only includes word/pos and proximity relations.
    \item \jeff{Semantic}: only includes word/pos and verb relations.
\end{enumerate}

\paragraph{Results}

Michael and Valiant note that they did not achieve any notable results from distributions 01 and 10 because the training and test procedures are misaligned.  They then present 5 different arrangements of the distributions (00 or 11) and the type of information (semantic, syntactic, or both).  The parameters and results were as follows.

\begin{table}[H]
\centering
\caption{Experimental Design Parameters}
\label{my-label}
\begin{tabular}{|l|l|l|}
\hline
Experiment & Enhanced Ruleset Distribution & Information Type \\ \hline
A          & 00                            & Semantic         \\ \hline
B          & 11                            & Semantic         \\ \hline
C          & 00                            & Syntactic        \\ \hline
D          & 00                            & Both             \\ \hline
E          & 11                            & Both             \\ \hline
\end{tabular}
\end{table}

\begin{figure}[H]
  \caption{5 experiments along enhanced ruleset distribution and information type}
  \centering
    \includegraphics[width=0.8\textwidth]{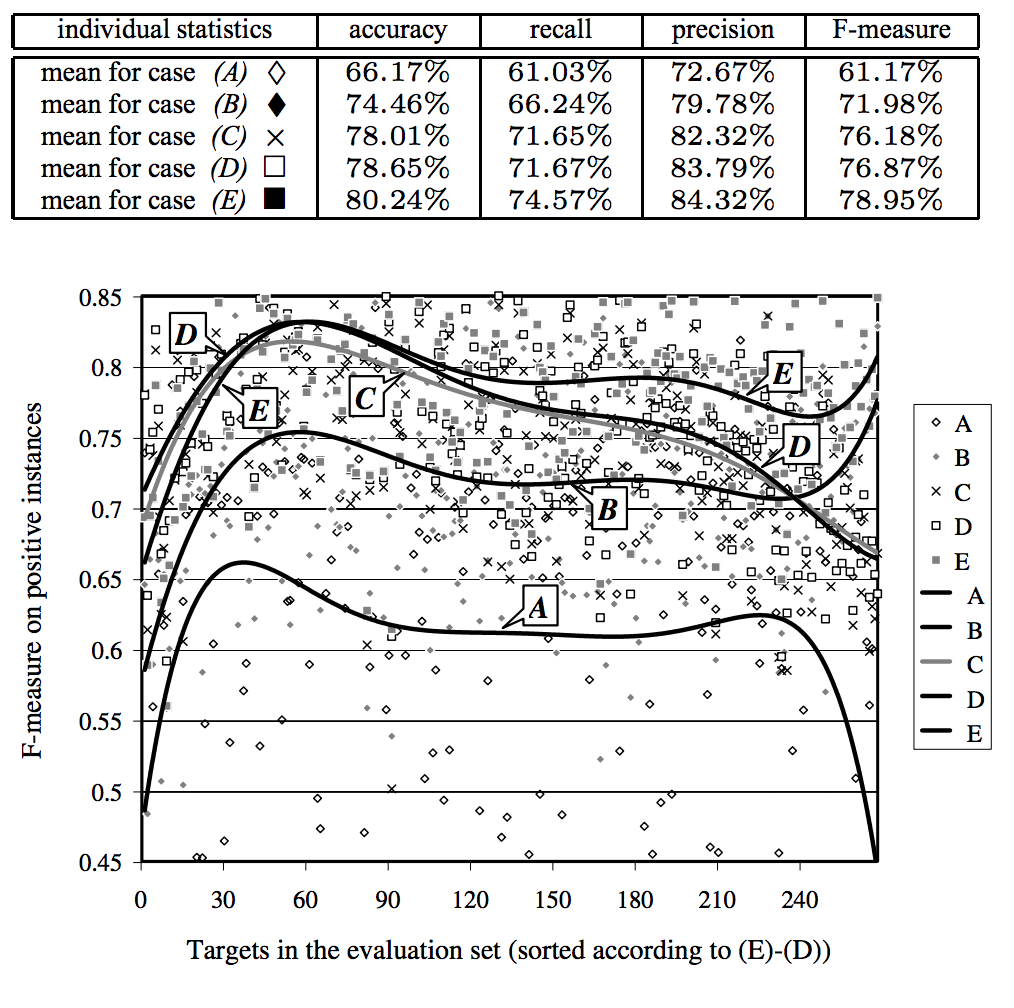}
    \includegraphics[width=0.8\textwidth]{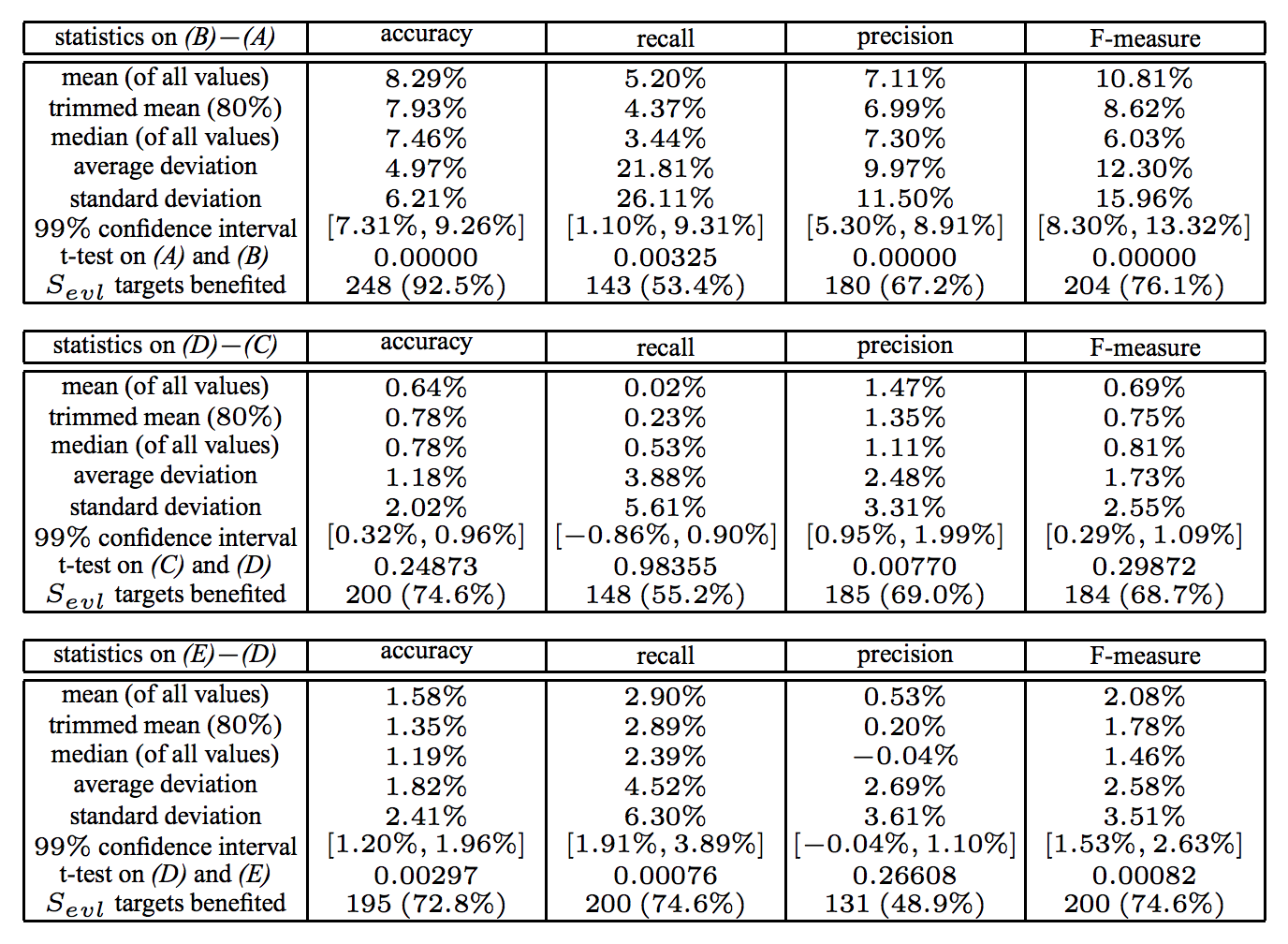}
\end{figure}

The best results from experiment $E$ are promising: the learned rule generates nearly 83\% accuracy on test scenes with potential relations and entities never seen before.  This strongly demonstrates the case for logical inference as a component of learning.

The cross-experiment comparisons demonstrate that:
\begin{enumerate}
    \item Syntactic information is critical for performance at any level (without it, performance does not significantly improve from 50\%).
    \item Semantic information in addition to syntactic information demonstrates little improvement -- in other words, verbs provide no additional information on proximity.
    \item An enhanced set of scenes from augmenting rules does not initially increase performance, but allows the algorithm to generalize to a larger set of targets in the evaluation task.
\end{enumerate}

Michael and Valiant conclude that chaining learned rules is both feasible and scalable.


\subsection{Are Rules Necessary?}

We conclude this section on PAC-frameworks to raise the question of whether learning logical induction requires an explicit representation of rules: Khardon and Roth \cite{khardon-l2r} created a framework that requires no such representation, while Michael and Valiants's Robust Logics and Knowledge Infusion frameworks centralizes the principle of explicit rules (and also require the form of the rules to be given).

Valiant even brings up this issue in his original Knowledge Infusion paper and brings up 3 reasons in defense of rules. \cite{valiant-infusion}
\begin{enumerate}
    \item Chaining rules allows the algorithm to make inference about situations too infrequent to have data that supports a learning-only deduction.
    \item Chaining in Robust Logics learns equivalence rules that allow for a higher complexity than direct learning.
    \item Learning without Robust Logics does not allow for programmed rules as priors.  There is no analog to this knowledge transfer.
\end{enumerate}

Juba argues that Valiant's knowledge infusion algorithm is not sufficiently powerful to support the finding that the number of rules that can be learned is exponentially large in the size of the knowledge base \textbf{because Valiant's algorithm preserves explicit rules}. \cite{juba}

We define 6 terms in anticipation of a theorem justifying Juba's claim.

\theoremstyle{definition}
\begin{definition}{(Partial example)}
An element of $\{0,1,\$\}^n$: analogous to Valiant's notion of an obscured scene.
\end{definition}

\theoremstyle{definition}
\begin{definition}{(Witnessed formula)}
A boolean formula $\varphi$ is witnessed in partial example $\rho$ iff all literals of $\varphi$ are unobscured in $\rho$.
\end{definition}

\theoremstyle{definition}
\begin{definition}{(Witnessed formula)}
A boolean formula $\phi$ is witnessed in partial example $\rho$ iff all literals of $\varphi$ are unobscured in $\rho$.

Given literals $L$ observed to be true and masked literals $M$, a threshold connective formula $\varphi=\text{sign}(c_1\phi_1+...+c)k\phi_k-b\geq 0)$ is witnessed in partial example $\rho$ iff:
\begin{align*}
    \sum_{l\in L} c_l + \sum_{m\in M} \min(0, c_m) \geq b 
\end{align*}
\end{definition}

\theoremstyle{definition}
\begin{definition}{(Restricted formula)}
Given literals $L$ observed to be true over $\rho$, a restricted formula $\varphi\mid_\rho$ over threshold connected $\varphi=\text{sign}(c_1\phi_1+...+c)k\phi_k-b\geq 0)$ is recursively defined as $\varphi$ if $\varphi$ is observed and $\text{sign}(\sum_{L^c} c_k - b)$ otherwise.
\begin{align*}
    \sum_{l\in L} c_l + \sum_{m\in M} \min(0, c_m) \geq b 
\end{align*}
\end{definition}

\theoremstyle{definition}
\begin{definition}{(Automatizability Problem)}
Given a set of proofs $S$ and an input formula $\varphi$, the automatizability problem is deciding whether a proof of $\varphi$ exists in $S$.
\end{definition}

\theoremstyle{definition}
\begin{definition}{(Restriction-closed set of proofs)}
A set $S$ is restriction-closed if for all formulas $\varphi$, the existence of a proof of $\varphi$ in $S$ implies the existence of a proof of $\varphi\mid_\rho$ in $S\mid_\rho$.
\end{definition}

Finally, we arrive at the theorem.

\begin{theorem}
Let $S$ be a restriction-closed set of proofs, and suppose there is an algorithm for the Automatizability Problem that is polynomial $P$ in $n$, $||\varphi||$, and $|H|$.  Then there is an algorithm that uses $O\left(\log\frac{1}{\delta}\right)$ examples and runs in $O\left(P\cdot \log\frac{1}{\delta}\right)$ time such that deciding which of the following holds is done correctly with probability $1-\delta$:
\begin{itemize}
    \item $H\implies \varphi$
    \item There exists a proof from $\varphi$ to $\{\varphi_1,...,\varphi_k\}\cup H$, where $\{\varphi_1,...,\varphi_k\}$ are all witnessed.
\end{itemize}
\end{theorem}

The proof is omitted as it is outside the scope of PAC-reasoning; for an explicit algorithm that accomplishes the above, see Juba's original publication. \cite{juba}

The consequence of Juba's theorem is that Juba finds a mechanism (witnessing) that preserves Valiant's finding that examples can support an exponential number rules, while the Knowledge Infusion algorithm falls short.

\section{Applied Approaches}

For a problem as difficult and general as reasoning, strong bounds (such as those specified by the PAC model) were previously only found in settings with many restrictions; for example, robust logics requires that the structure and relational ontology of all rules be known beforehand, where that very structure might be the most difficult aspect to learn.

The advent of deep learning and differentiable techniques has allowed for more complex functional approximations, so we now explore some recent advancements on the PAC-findings using applied approaches.

Whereas the PAC-findings tended to build on each others' theorems, these deep learning findings are all largely the same in format.
\begin{enumerate}
    \item Theorize a model.
    \item Formulate an experiment.
    \item Present empirical results.
\end{enumerate}

We now present findings in symbolic ontology and symbolic reasoning in this format.

\subsection{Learning Symbolic Ontology}

This corresponds to the problem of scene construction in knowledge infusion.

\subsubsection{Deep Symbolic Reinforcement Learning}

Garnelo, Arulkumaran, and Shanahan present findings on inferring objects from frame-by-frame time-varying image data. \cite{garnelo-symbolic}

\paragraph{Model}

The premise is that many of the deep learning issues addressed (lack of priors, lack of transfer, data-hungriness, transparency) are solved if the learning is done on data that is symbolically represented as ILP.  Thus, the model learns a symbolic encoding of the data, then learns deeply on the symbolic representations.

\begin{figure}[H]
  \caption{Left: the proposed deep symbolic framework.  Right: an example ``scene'' (in Valiant's terminology) of the agent (+) amongst its environment.}
  \centering
    \includegraphics[width=0.5\textwidth]{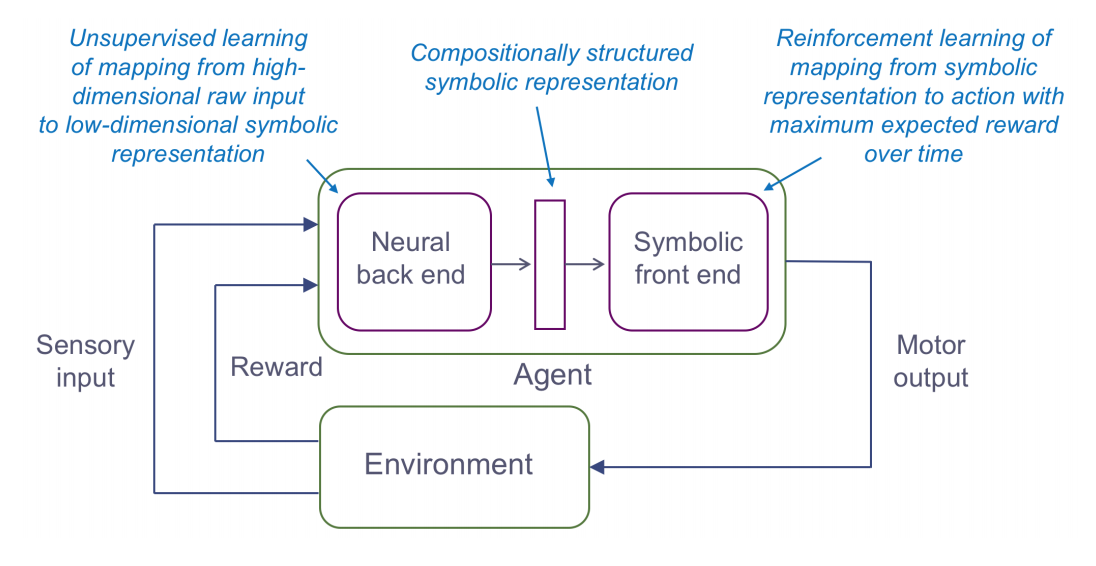}
    \includegraphics[width=0.3\textwidth]{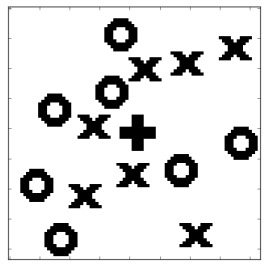}
\end{figure}

\paragraph{Experiment}

Given ``responsive'' video data of scenes as depicted in Figure 3 and an existing ontology $A$ of objects, the procedure is as follows.

\begin{enumerate}
    \item For each frame, learn which symbols in $A$ are in the picture and where.
    \item For the sequence of frames, learn how specific instances of symbols in $A$ persist from frame-to-frame and represent ``motion.''
    \item Do reinforcement learning on the environment: colliding with certain objects yields positive rewards, while colliding with other objects yields negative rewards.
\end{enumerate}

Each subproblem is solved with a heuristic algorithm.  We present the first in its entirety and the others briefly because they are outside the scope of this review.

\begin{algorithm}[H]
\caption{Symbol Generation Algorithm}\label{alg:symbol-gen}
\begin{algorithmic}[1]
\Procedure{Symbol-Generation}{Picture $P$, CNN $C$}\Comment{Learning a symbols from a pixel map}
\For{Pixel $p$ in $P$}
\State Convolute $C$ over $p$.
\State Find the symbol $\sigma_p$ that best represents $p$
\EndFor
\State Threshold pixels by max activation and create subset $S$ of pixels that surpass the threshold.  Ideally, these each pixel is a ``representative'' pixel and has a one-to-one correspondence with an actual object.
\For{Pixel $p$ in $S$}
\State Extract final symbols by comparing the spectra of activations in $p$ against the known spectra of activations for each object in $A$.
\EndFor
\EndProcedure
\end{algorithmic}
\end{algorithm}

Object persistence is modeled with a transition probability matrix between each pair of consecutive frames, where the authors hardcode 2 priors.

\begin{itemize}
    \item Spatial proximity: the likelihood is defined as the inverse distance $L_{dist}=\frac{1}{d}$.
    \item Neighborhood: between frames, the number of nearby objects is likely to be similar: $L_{neigh}=\frac{1}{1+\Delta N}$, where $\Delta N$ is the change in the number of neighbors.
\end{itemize}

Given the spatio-temporal representation of an object, the authors then implement a reinforcement learning algorithm using tabular Q-learning.  The policy update rule is:
\begin{align*}
    Q^{ij}(s_t^{ij},a_t)\gets Q^{ij}(s_t^{ij},a_t)+\alpha \left(r_{t+1}+\gamma(\max_a Q^{ij}(s_{t+1}^{ij},a)-Q^{ij}(s_t^{ij},a_t))\right)
\end{align*}

where $Q$ is the policy, $i,j$ are types of objects, $s^{ij}$ is a state of interaction betwen objects of types $i,j$, $a_t$ is an action, $r_{t+1}$ is the reward, and $\gamma$ is a temporal discount factor.  A less dense articulation of this policy update is that we update the policy proportional to the difference of the reward and the hypothesized best course of action.

\paragraph{Results}

The authors compare the performance of their learned symbolic policy to Deep Q-Networks (DQN), Deepmind's state-of-the-art reinforcement learning module.  The results are depicted below.

\begin{figure}[H]
  \caption{Deep Symbolic Learning performance, with DQN as a benchmark}
  \centering
    \includegraphics[width=\textwidth]{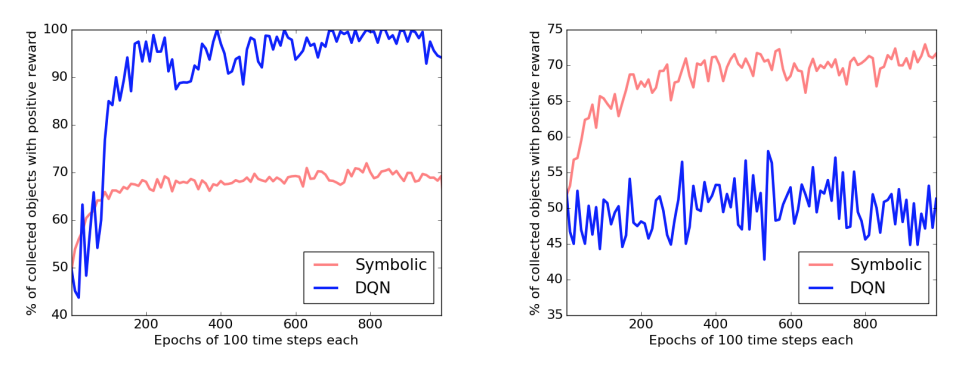}
\end{figure}

While DQN achieves optimal performance in a fixed grid setting very quickly and Deep Symbolic Learning plateaus early on, DQN fails to generalize to the random setting, while Deep Symbolic Learning does not suffer at all (curiously, it seems to perform better on the randomized setting).

We conclude that learning a symbolic representation aids in generalization in light of small amounts of data.

\subsection{Learning to Reason over a Given Ontology}

This corresponds to the rule inference problem in knowledge infusion.  While Deep Symbolic Learning made some advancement in procuring the encoding of data into logical variables, its learning algorithm did not take advantage of the symbolic structure.  The following findings are all algorithms to learn given fixed objects and relations in the ILP setting.

\subsubsection{Differentiable Inductive Logic Programming ($\partial$ILP)}
\paragraph{Model}

$\partial$ILP is a reimplementation of ILP where the rules are end-to-end differentiable (and thus gradient descent techniques can be used for learning). \cite{evans-explanatory}

\theoremstyle{definition}
\begin{definition}{(Valuation)}
Given a set $G$ of $n$ ground atoms, a valuation is a vector $v\in [0,1]^n$ mapping each ground atom $\gamma_i$ to the real unit.  Each valuation element is a ``confidence.''
\end{definition}

Consider a set $\Lambda$ of atom-label pairs derived from positive examples $P$ and negative examples $N$:
\begin{align*}
    \Lambda=\{(\gamma,1)\mid \gamma\in P\} \cup \{(\gamma,1)\mid \gamma\in P\}
\end{align*}

The interpretation is that a pair $(\gamma,\lambda)$ represents some atom $\gamma$ and the ``truthiness'' of $\gamma$ (1 for entirely positive and 0 for entirely negative).

We now have the setting for constructing a likelihood.  Recall that an ILP problem is defined in terms of $L,G,P,N$, where $L$ is a language frame (a target relation atom and a set of predicates), $B$ is a set of ground atoms ($B\subseteq G$), $P$ are positive examples, and $N$ are negative examples.  We further specify a set of clause-weights $W$.  The likelihood of $\lambda$ for a given atom $\gamma$ is then:
\begin{align*}
    p(\lambda\mid \gamma, W, L, B)
\end{align*}
We break down this likelihood into 4 functions.
\begin{align*}
    p(\lambda\mid \gamma, W, L, G)&=f_{extract}(f_{infer}(f_{convert}(B),f_{generate}(L),W,T), \alpha)
\end{align*}
\begin{itemize}
    \item $f_{extract}$ takes a valuation $x$ and an atom $\gamma$ and extracts the value $x[\gamma]$.
    \item $f_{convert}$ is an indicator r.v. for whether an ground atom from $G$ is in $B$.
    \item $f_{generate}$ produces a set of clauses from the language frame: $f_{generate}(L)=\{cl(\tau_p^i\mid p\in P_i,i\in\{1,2\}\}$, where $cl$ are the clauses the satisfy the specified template.
    \item Finally, $f_{infer}$ is a mapping $[0,1]^n\times C\times W\times \mathbb{N}\rightarrow [0,1]^n$ which infers using forward-chains from the generated clauses under the clause weights $W$.
    
    These $W$ are defined as follows.  Let $W=\{W_1,...,W_{|P_i|}\}$ (one for each predicate).  A particular weight $W_p[j,k]$ represents how strongly the system believes that the pair of clauses $(C_p^{1,j}, C_p^{2,k})$ is the ``correct way'' to infer predicate $p$.  We enforce this definition with a softmax approach.
    \begin{align*}
        W_p^*[j,k]=\frac{e^{W_p[j,k]}}{\sum_{j',k'} e^{W_p[j',k']}}
    \end{align*}
    
    In order to apply these weights to a predicate $p$, we compute all possible pairs of clauses that could infer $p$, then sum a weighted average of their softmax weights:
    
    \begin{align*}
        b_t^p=\sum_{j,k} c_t^{p,j,k}\cdot \frac{e^{W_p[j,k]}}{\sum_{j',k'} e^{W_p[j',k']}}
    \end{align*}
    
    where $b_t^p$ is the confidence in $p$ with $t$ steps of chained inference and $c_t^{p,j,k}$ is the contribution to the confidence of the inference pathway $p\gets j,k$.
    
    This algorithm is carried out by a process called amalgamation; the details of this algorithm are omitted as they are outside the scope of this review, but they can be found in section 4.4 of the original publication. \cite{evans-explanatory}
    
    Note that in order to work with 2D matrices (as opposed to $n$-D tensors), the number of predicates in a rule is capped at two (much like Valiant's Knowledge Infusion algorithm).
    
    We will learn the weights $W$ by gradient descent.
\end{itemize}

The architecture is depicted below.

\begin{figure}[H]
  \caption{Deep Symbolic Learning performance, with DQN as a benchmark}
  \centering
    \includegraphics[width=0.3\textwidth]{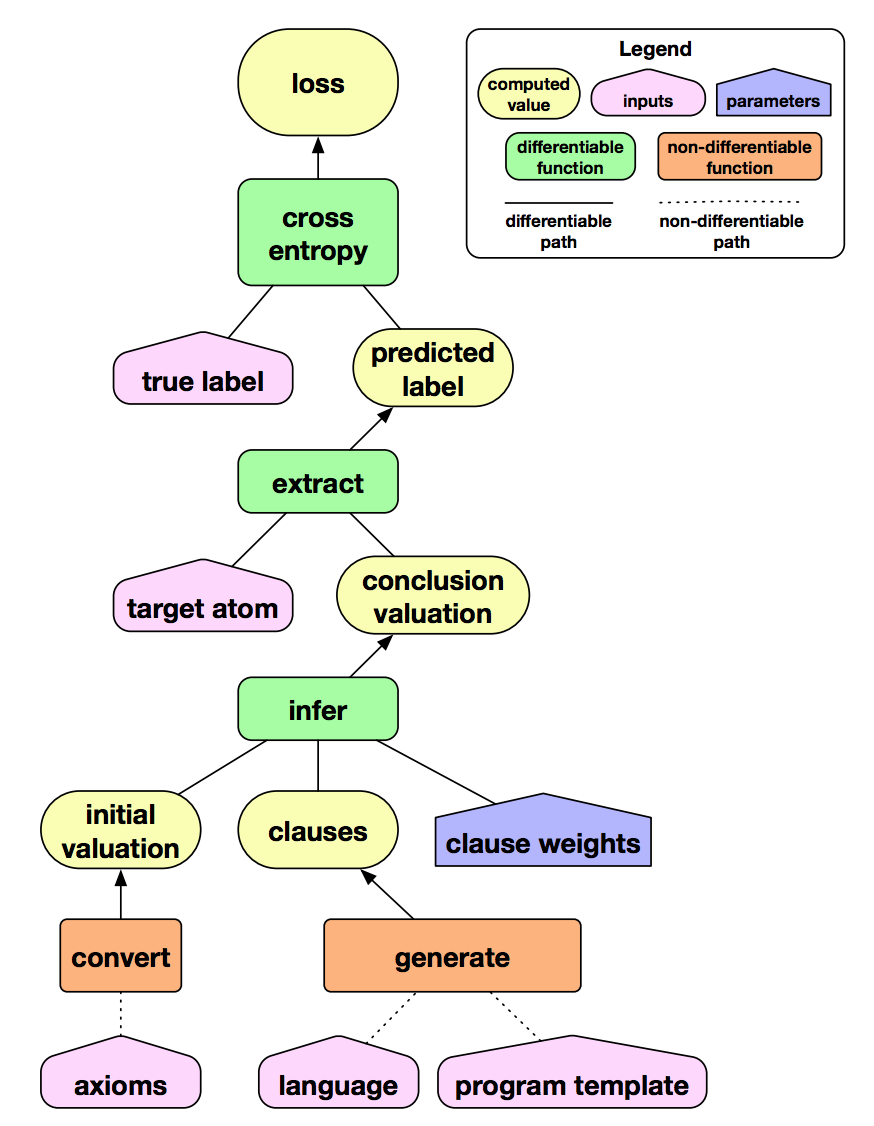}
\end{figure}

We fit this model by minimizing the expected log-likelihood over the dataset $\Lambda$ that we defined above.  Explicitly, that log-likelihood is:
\begin{align*}
    loss=-E_{\alpha,\lambda\in \Lambda}\left[\lambda\cdot \log p(\lambda\mid \gamma, W, L, G) + (1-\lambda)\cdot \log(1-p(\lambda\mid \gamma, W, L, G))\right]
\end{align*}

\paragraph{Experiment}

Evans and Grefenstette first tested $\partial$ILP on 20 tasks with no noise, including learning even numbers, learning FizzBuzz (divisibility by 3 xor 5), and learning graph cyclicity.  Pending success, they then tested on ambiguous data such as learning even numbers from pixel images, learning if one image is exactly two less than another, learning if at least one image is a 1, learning the less-than relation from images, etc.
\paragraph{Results}

Results for unambiguous data were shown to be quite strong, beating state-of-the-art results.  Results for ambiguous data were even more impressive, tolerating up to 20\% noise with very little gain in mean squared-error.

\begin{figure}[H]
  \caption{$\partial$ILP performance}
  \centering
    \includegraphics[width=0.4\textwidth]{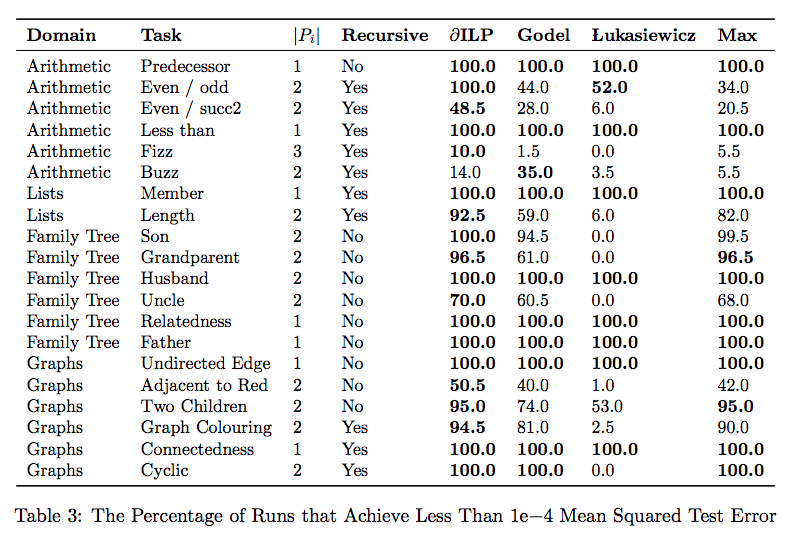}
    \includegraphics[width=0.5\textwidth]{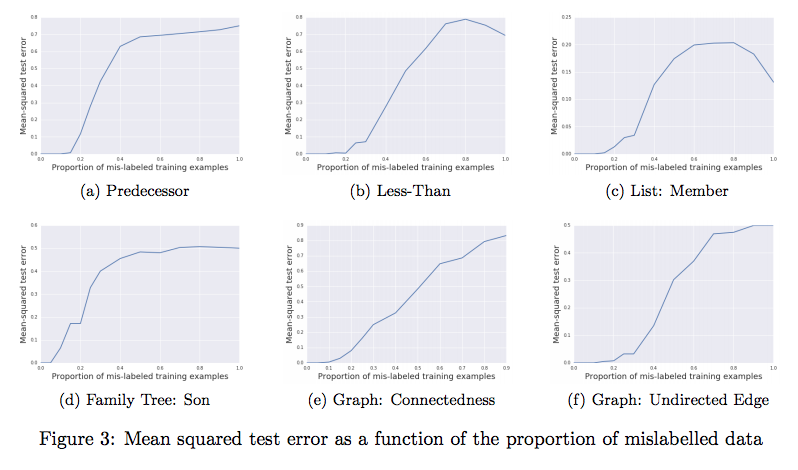}
\end{figure}

\subsubsection{DeepLogic}

Whereas $\partial$ILP had 4 models, of which 1 was a polynomial-growing framework of transition weights, DeepLogics seeks to compute all steps of inference using a fixed number of neural networks.

\paragraph{Model}

The goal is to learn some function $f:C\times Q\rightarrow \{0,1\}$, where $C$ is the context of the logic program and $Q$ is the query (head atom in ILP terms).  Cingillioglu et al. present a module called a Neural Inference Network (NIN) to accomplish this goal.

NIN takes as input two sequences of characters $c_0^C,...,c_m^C$ (context vector) and $c_0^Q,...,c_n^Q$ (query vector).  Given the existing structure of $R$ rules in the ILP framework, we can create an input tensor $I^C\in \mathbb{N}^{R\times L\times m'}$, where $L$ is the number of literals and $m'$ is the length of these literals.

The NIN then uses a gated recurrent unit (GRU) on the one-hot encoding of these literals to create an embedding of the literals in $\mathbb{R}^d$.  The NIN then computes an attention vector over these featurized literals using a MLP, then finally applies the rule to the featurized literals weighted by attention using another RNN.\footnote{In full disclosure, the authors describe these mechanisms in notation-dense formulation but never define their notation, so some context has been inferred.}

The architecture of the system is depicted below.

\begin{figure}[H]
  \caption{NIN architecture}
  \centering
    \includegraphics[width=\textwidth]{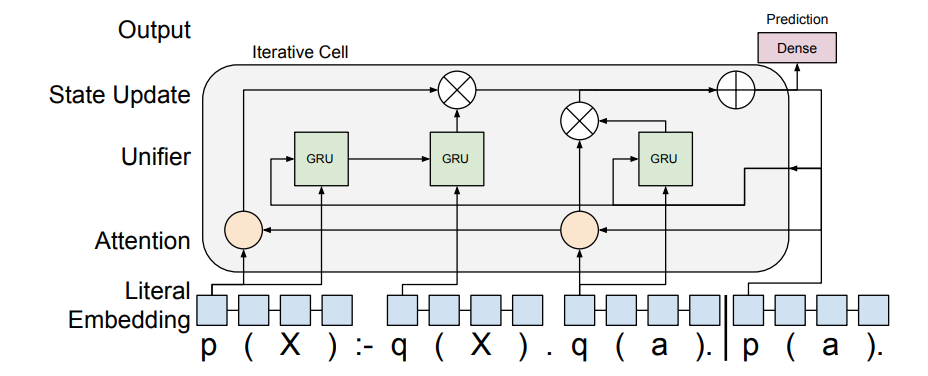}
\end{figure}

\paragraph{Experiment}

Cingillioglu et al. train over some basic logical inference tasks.
\begin{itemize}
    \item Facts (ground atoms with no variables)
    \item Unification (rules with empty bodies and an atom with a variable in the head)
    \item N-step deduction
    \item Logical AND
    \item Logical OR
    \item Transitivity
    \item N-step deduction with negation
    \item Logical AND with negation
    \item Logical OR with negation
\end{itemize}

\paragraph{Results}

The results are contrasted with the performance of a dynamic memory network (DMN), the state-of-the-art question-answering architecture.

We see that the NIN performs very well as the number of deduction steps increases (it was, after, architected for multi-hop reasoning).  However, as we say earlier with Valiant's Knowledge Infusion algorithm and with $\partial$ILP, the length of predicates is extremely problematic for inference.

\begin{figure}[H]
  \caption{NIN results}
  \centering
    \includegraphics[width=\textwidth]{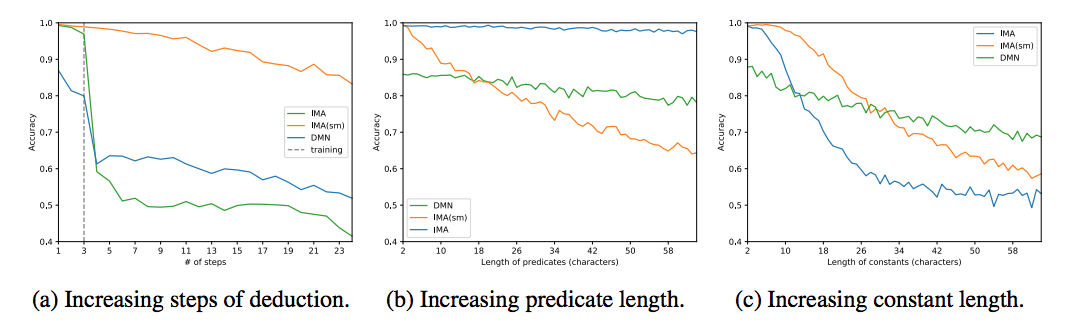}
\end{figure}

\section{Conclusions}
\subsection{Summary of Findings}

\begin{itemize}
    \item Learning without any ontology (in particular, if expressions are relaxed from booleans to arbitrary ontologies) is intractable. Negative result via conceptual graphs.  \cite{jappy-conceptual}

    \item Rules are PAC-learnable from scenes of the world given the structure of rules (the identities of the relation atoms in the predicates).  Positive result via robust logics. \cite{valiant-robust}
    
    \item The difficulty of learning rules grows \textbf{exponentially} with the maximum arity of the relation atoms.  Negative result due to robust logics, $\partial$ILP, and DeepLogic's NIN.  This problem looks to be especially difficult as the intractability was discovered by at least 3 independent sources from 3 different fields (PAC learning, ILP, and deep learning).  \cite{valiant-robust} \cite{evans-explanatory} \cite{cingillioglu-deeplogic}
    
    \item Given a learning oracle that provides counterexamples of incorrect rules, it is possible to learn exact-reasoning and PAC-reasoning algorithms for problems that are otherwise NP-hard. Positive result via the learning-to-reason framework. \cite{khardon-l2r}
    
    \item Learning many rules over an ontology is not much harder than learning a single rule over the same ontology. Positive result via knowledge infusion. \cite{valiant-infusion}
    
    \item Parallel learning via knowledge infusion over the robust logics framework yields successful applied results on natural language, especially when syntatical structures are provided in lieu of relations. Positive result via knowledge infusion. \cite{michael-infusion}
    
    \item Learning over a symbolic representation of data rather than raw information helps generalize to non-trivial distributions. Positive result via deep symbolic learning. \cite{garnelo-symbolic}
    
    \item Relaxing booleans from $\{0,1\}$ to the unit interval $[0,1]$ allows for a high-performing differentiable framework. Positive result via $\partial$ILP \cite{evans-explanatory}
    
    \item Featurizing literals (recognizing that objects and relations can be represented more meaningfully than a simple index) allows neural nets to reason effectively.  Positive result via DeepLogic's NIN. \cite{cingillioglu-deeplogic}
\end{itemize}

\subsection{Future Work}
\subsubsection{Integrated Models}

The most obvious way to progress from these findings is to recognize that many of these works have fortuitously designed their models as modular units.  Thus, a comprehensive end-to-end system can be constructed by stacking these modules.

\begin{table}[H]
\centering
\caption{Functional Reasoning Modules}
\label{my-label}
\begin{tabular}{|p{8cm}|p{6cm}|}
\hline
Module                                        & Implementations    
                        \\ \hline
Creating a Knowledge Base                            & Knowledge Infusion  \\ \hline
Extracting Symbols                            & Deep Symbolic Learning                                       \\ \hline
Featurizing Literals                          & DeepLogic                                                    \\ \hline
Reasoning from Literals without Rule Template & Learn-to-Reason, $\partial$ILP, DeepLogic                    \\ \hline
Reasoning from Literals with Rule Template    & Robust Logics, Knowledge Infusion (form given by POS tagger) \\ \hline
\end{tabular}
\end{table}

For example, consider the problem of reasoning about flying through an asteroid field.  We might stack these modules into system that uses Knowledge Infusion to generate a knowledge base (e.g. asteroids have momentum), Deep Symbolic Learning to extract symbols from video data (e.g. brown pixel $\implies$ asteroid), DeepLogic to featurize the literals (e.g. asteroids are bad), and $\partial$ILP to learn rules over the featurized literals (e.g. left-thrust to avoid asteroid from the right).

\subsubsection{Unsolved Problems}

\paragraph{Encoding Priors}
Many of these studies note (and rightly so) that a logical rule form allows one to encode priors.  However, none of these studies actually provide a methodology for doing so; this is non-trivial especially since the 2 applied differentiable reasoning systems (NIN and $\partial$ILP) have other parameters associated with rules (confidence, attention) that must be learned.

\paragraph{Flexible Object Ontology}
Similarly, many of these studies note (and rightly so) that a logical rule form makes transfer-learning a theoretical possibility.  However, all of the systems are hardcoded in terms of a fixed number of objects and relations.  Suppose we gave any of these systems a dataset to learn rules about fruits.  Suppose further that we then acquired data about legumes.  There would be no way to transfer-learn under any of the given systems -- robust logics requires the ontology in advance, and the deep systems are all hardcoded in input size.

\paragraph{Arbitrary Arity}
Of the 7 studies surveyed with direct positive results, 3 made no attempt to learn rules of arity greater than 2 (or any mention of such a possibility).  Two noted the 
intractability of such a problem (Valiant \cite{valiant-robust}, Evans \cite{evans-explanatory}), and one empirically tested and failed over longer rules (Cingillioglu \cite{cingillioglu-deeplogic}).  Finally, even given a Reasoning Oracle that provides counterexamples (not necessarily a feasibility in an applied setting), the Learn-to-Reason algorithm only provides positive results given a log-bound on the number of literals per clause.

The unscalability in arity is the only direct shortcoming that all of these studies have in common.

\subsubsection{Applications}

\paragraph{NLP} The most straightforward application seems to be natural language, which happens to be the domain Michael and Valiant chose for Knowledge Infusion.  This is because natural language is rich enough that learning dynamic rules is useful, while structured enough to provide a given ontology of objects and relations (nouns and verbs/proximity clauses).  Furthermore, many state-of-the-art dialogue systems in industry (such as Cognea, Watson Assistant, IPSoft, Google DialogFLow) have rules-based entity detection.  These rules provide a potential knowledge base of \textbf{priors} that can potentially be encoded into a bootstrapping system.

\paragraph{Econometric Causality} A significant field of study in econometrics is identifying \textbf{instrumental variables} (IVs) because a correlation from IVs can lead to a sound conclusion of causality.  Findings in AI reasoning systems can supplement IV research since the learned form (rules as implications) gives a direct representation of causality (and is robust to multi-hop reasoning, as shown by DeepLogic's NIN).

\paragraph{Ethics} Recent European legislation (in particular, the General Data Protection Regulation, GDPR) has made data transparency a worldwide priority.  Compliance with GDPR requires that companies be able to explain the behavior of their models in the context of protected classes.  While this may be straightforward in generalized linear models, this is highly non-trivial for black-box models (as most deep learning systems are).  AI reasoning systems can help sustain the high performance of deep learning systems while maintaining compliance with ethics and regulation.

\printbibliography

\end{document}